%% file: main.tex
\documentclass[letterpaper, 10 pt, conference]{ieeeconf}  

\IEEEoverridecommandlockouts                              
\include{package}

\include{header}

\include{notation}
\usepackage{mysymbol}
\input{macros.tex}

\title{\LARGE \bf
Technical Report: Reactive Planning for Mobile Manipulation Tasks\\ in Unexplored Semantic Environments
}

\author{Vasileios Vasilopoulos*, Yiannis Kantaros*, George J. Pappas and Daniel E. Koditschek
\thanks{* Equal contribution.}%
\thanks{The authors are with the GRASP Lab, University of Pennsylvania, Philadelphia, PA 19104.
        {\tt\small \{vvasilo, kantaros, pappasg, kod\}@seas.upenn.edu}}%
\thanks{This work was supported by AFRL grant FA865015D1845 (subcontract 669737-1), AFOSR grant FA9550-19-1-0265 (Assured Autonomy in Contested Environments), and ONR grant \#N00014-16-1-2817, a Vannevar Bush Fellowship held by the last author, sponsored by the Basic Research Office of the Assistant Secretary of Defense for Research and Engineering.}%
}

\begin{document}

\maketitle
\thispagestyle{empty}
\pagestyle{empty}

\begin{abstract}
Complex manipulation tasks, such as rearrangement planning of numerous objects, are combinatorially hard problems. Existing algorithms either do not scale well or assume a great deal of prior knowledge about the environment, and few offer any rigorous guarantees. In this paper, we propose a novel hybrid control architecture for achieving such tasks with mobile manipulators. On the discrete side, we enrich a temporal logic specification with mobile manipulation primitives such as moving to a point, and grasping or moving an object. Such specifications are translated to an automaton representation, which orchestrates the physical grounding of the task to mobility or manipulation controllers. The grounding from the discrete to the continuous reactive controller is online and can respond to the discovery of unknown obstacles or decide to push out of the way movable objects that prohibit task accomplishment. Despite the problem complexity, we prove that, under specific conditions, our architecture enjoys provable completeness on the discrete side, provable termination on the continuous side, and avoids all obstacles in the environment. Simulations illustrate the efficiency of our architecture that can handle tasks of increased complexity while also responding to unknown obstacles or unanticipated adverse configurations.
\end{abstract}


\input{1-introduction.tex}
\input{2-problem-formulation}

\input{3-ltl-planner-icra}
\input{4-logic-reactive-interface}
\input{5-reactive-planner}
\input{6-simulations}
\input{7-conclusion}

\appendices

\input{appendix_LTL}
\input{appendix_reactive_planner}





\bibliographystyle{IEEEtran}
\bibliography{IEEEabrv,references, YK_bib}

\end{document}

%% file: package.tex
\usepackage{enumerate}
\usepackage[cmex10]{amsmath} 
\usepackage{amssymb} 
\usepackage{amsfonts}

\usepackage{amsthm}

\usepackage{graphicx}
\usepackage{graphics} 
\usepackage{epsfig} 
\usepackage{accents}
\usepackage{algorithm}
\usepackage{algpseudocode}
\usepackage{subcaption}

\algdef{SE}[DOWHILE]{Do}{doWhile}{\algorithmicdo}[1]{\algorithmicwhile\ #1}%

\usepackage{color}
\makeatletter
\let\NAT@parse\undefined
\makeatother

\usepackage{cite}

\usepackage{nohyperref}
\usepackage{url}
\usepackage{comment}


%% file: header.tex
	\newtheoremstyle{myplain}
	  {}
	  {}
	  {\itshape}
	  {}
	  {\bfseries}
	  {}
	  {5pt plus 1pt minus 1pt}
	  {}

	\newtheoremstyle{mydefinition}
	  {}
	  {}
	  {\normalfont}
	  {}
	  {\bfseries}
	  {}
	  {5pt plus 1pt minus 1pt}
	  {}
  
	\theoremstyle{myplain}

	\newtheorem{proposition}{Proposition}

	\theoremstyle{mydefinition}
	\newtheorem{definition}{Definition}

   \newcommand{\field}[1]{\mathbb{#1}}
   \newcommand{\reals}{\field{R}}

\algdef{SE}[DOWHILE]{Do}{doWhile}{\algorithmicdo}[1]{\algorithmicwhile\ #1}%


%% file: macros.tex
\newcommand{\robotradius}{r} 
\newcommand{\sensorrange}{R} 

\newcommand{\robotposition}{\mathbf{x}} 

\newcommand{\robotpositionunicycle}{\overline{\mathbf{x}}} 
\newcommand{\robotorientation}{\psi} 

\newcommand{\goalposition}{\mathbf{x}^*} 

\newcommand{\workspace}{\mathcal{W}} 
\newcommand{\enclosingfreespace}{\mathcal{F}_e} 
\newcommand{\freespace}{\mathcal{F}} 

\newcommand{\obstacleset}{\tilde{\mathcal{O}}} 
\newcommand{\obstaclesetdilated}{\mathcal{O}} 

\newcommand{\knownobstacleset}{\tilde{\mathcal{P}}} 

\newcommand{\knownobstaclesetdilated}{\mathcal{P}} 

\newcommand{\unknownobstacleset}{\tilde{\mathcal{C}}} 

\newcommand{\unknownobstaclesetdilated}{\mathcal{C}} 

\newcommand{\movableobject}{\tilde{M}} 
\newcommand{\movableobjectset}{\tilde{\mathcal{M}}} 

\newcommand{\movableobjectdilated}{M} 
\newcommand{\movableobjectgoal}{\mathbf{x}^*} 
\newcommand{\movableobjectsetdilated}{\mathcal{M}} 
\newcommand{\movableobjectcardinality}{N_M} 
\newcommand{\movableobjectradius}[1]{\rho_{#1}} 

\newcommand{\diffeogeneric}{\mathbf{h}} 
\newcommand{\diffeo}{\mathbf{h}^{\hybridmode}} 

\newcommand{\controlfullyactuated}{\mathbf{u}} 
\newcommand{\controlunicycle}{\overline{\controlfullyactuated}} 
\newcommand{\linearinput}{v} 
\newcommand{\angularinput}{\omega} 

\newcommand{\hybridmode}{\mathcal{I}} 

\newcommand{\ballclosure}[2]{\overline{\mathsf{B}(#1,#2)}}

%% file: 1-introduction.tex
\section{INTRODUCTION}
\label{sec:introduction}

\begin{figure}[t]
\captionsetup{width=\linewidth,font=footnotesize}
\centering
\includegraphics[width=0.7\columnwidth]{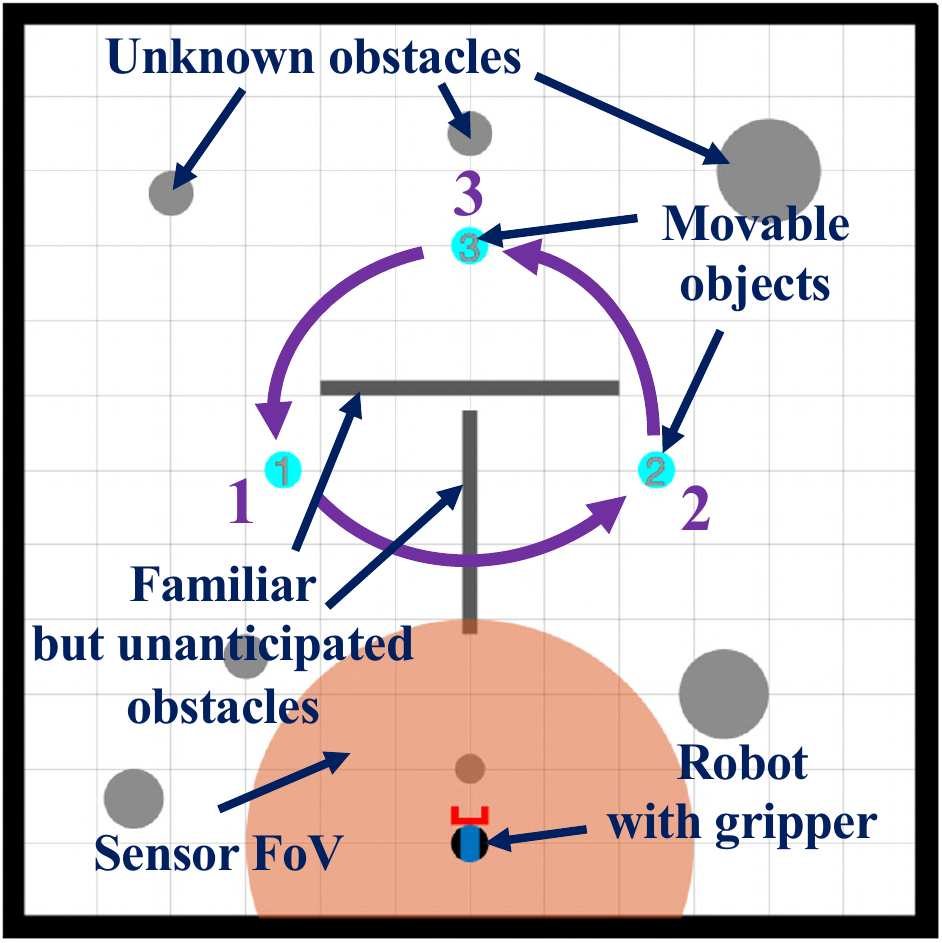}
\caption{An example of a task considered in this paper, whose execution is depicted in Fig.~\ref{fig:simulation_rearrangement}. A differential drive robot, equipped with a gripper (red) and a limited range onboard sensor for localizing obstacles (orange), needs to accomplish a mobile manipulation task specified by a Linear Temporal Logic (LTL) formula, in a partially known environment (black), cluttered with both unanticipated (dark grey) and completely unknown (light grey) fixed obstacles. 
Here the task is to rearrange the movable objects  counterclockwise, in the presence of the fixed obstacles. Objects' abstract locations (relative to abstract, named regions of the workspace) are known by the symbolic controller both \`a-priori and during the entire task sequence.  Geometrically complicated obstacles are assumed to be familiar but unanticipated in the sense that neither their number nor placement are known in advance. Completely unknown obstacles are presumed to be convex. All obstacles and disconnected configurations caused by the movable objects are handled by the reactive vector field motion planner (Fig.~\ref{fig:system}) and never reported to the symbolic controller.
}
\label{fig:front_figure}
\vspace{-18pt}
\end{figure}

Task and motion planning for complex manipulation tasks, such as rearrangement planning of multiple objects, has recently received increasing attention \cite{garrett_ijrr_2018,vega-brown2016asymptotically,Krontiris-RSS-15}. However, existing algorithms are typically combinatorially hard and do not scale well, while they also focus mostly on \textit{known} environments \cite{srivastava2014combined,he2015towards}. As a result, such methods cannot be applied to scenarios where the environment is initially unknown or needs to be reconfigured to accomplish the assigned mission and, therefore, online replanning may be required \cite{Kaelbling2011}, resulting in limited applicability.

Instead, we propose an architecture for addressing complex mobile manipulation task planning problems which can handle unanticipated conditions in unknown environments. Fig.~\ref{fig:front_figure} illustrates the problem domain and scope of our algorithm. Fig.~\ref{fig:system} illustrates the structure of the architecture and the organization of the paper. Our extensive simulation study suggests that this formal interface between a symbolic logic task planner and a physically grounded dynamical controller - each endowed with their own formal properties - can achieve computationally efficient rearrangement of complicated workspaces, motivating work now in progress to give conditions under which their interaction can guarantee success.

\begin{figure}[t]
\captionsetup{width=1.0\linewidth,font=footnotesize}
\centering
\includegraphics[width=1.0\linewidth]{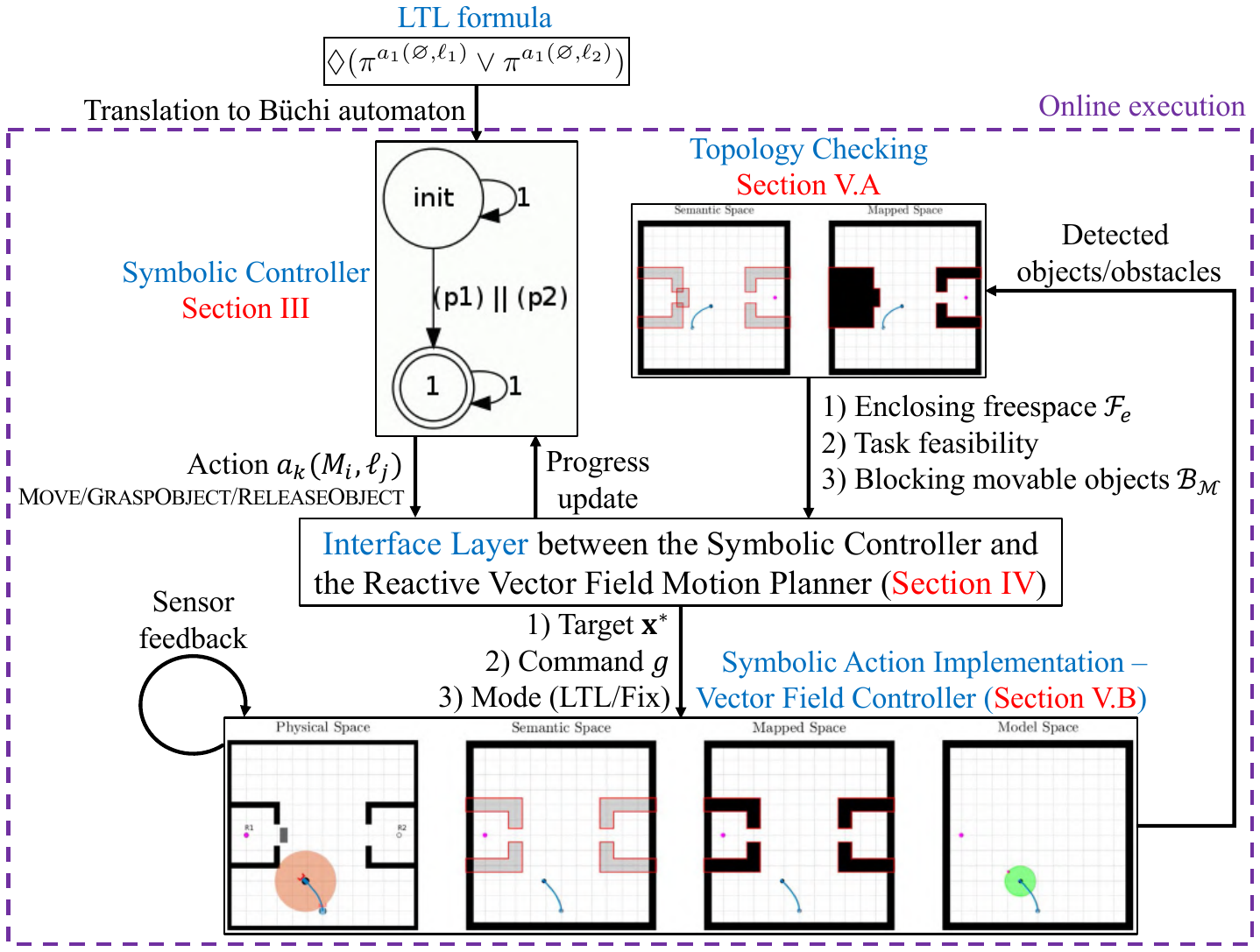}
\caption{System architecture: The task is encoded in an LTL formula, translated offline to a B\"uchi automaton (symbolic controller - Section~\ref{sec:ltl_planner}). Then, during execution time in a previously unexplored semantic environment, each individual sub-task provided by the B\"uchi automaton is translated to a point navigation task toward a target $\goalposition$ and a gripper command $g$, through an interface layer (Section~\ref{sec:logic_reactive_interface}). This task is executed online by realizing each symbolic action (Section~\ref{subsec:action_implementation}) using a reactive, vector field motion planner (continuous-time controller, \cite{vasilopoulos_pavlakos_bowman_caporale_daniilidis_pappas_koditschek_2020}) implementing closed-loop navigation using sensor feedback and working closely with a topology checking module (Section~\ref{subsec:topology_checking}), responsible for detecting freespace disconnections. The reactive controller, nominally in an {\it LTL mode}, guarantees collision avoidance and target convergence when both the initial and the target configuration lie in the same freespace component. On the other hand, if the topology checking module determines that the target is not reachable, the reactive controller either attempts to connect the disconnected configuration space by switching to a {\it Fix mode} and interacting with the environment to rearrange blocking movable objects, or the interface layer reports failure to the symbolic controller when this is impossible and requests an alternative action.}
\label{fig:system}
\vspace{-14pt}
\end{figure}

\subsection{Mobile Manipulation of Movable Objects}
Planning the rearrangement of movable objects has long been known to be algorithmically hard (PSPACE-hardness was established in \cite{Hopcroft_Schwartz_Sharir_1984}), and most approaches have focused on simplified instances of the problem. For example, past work on reactive rearrangement using vector field planners such as navigation functions \cite{rimon1992} assumes either that each object is actuated \cite{whitcomb_koditschek_cabrera,Karagoz_Bozma_Koditschek_2014} or that there are no other obstacles in the environment \cite{Bozma_Koditschek_2001,Karagoz_Bozma_Koditschek_2004,Arslan_Guralnik_Koditschek_2016}. When considering more complicated workspaces, most approaches focus either on sampling-based methods that empirically work well \cite{vandenBerg2010}, motivated by the typically high dimensional configuration spaces arising from combined task and motion planning \cite{garrett_ijrr_2018,Krontiris-RSS-15}, or learning a symbolic language on the fly \cite{konidaris2018}. Such methods can achieve asymptotic optimality \cite{vegabrown_2018} by leveraging tools for efficient search on large graphs \cite{Wolfe2010}, but come with no guarantee of task completion under partial prior knowledge and their search time grows exponentially with the number of movable pieces \cite{vegabrown_rss2017}. Sampling-based approaches have also been applied to the problem of navigation among movable obstacles (NAMO) \cite{stilman_kuffner_2006}, where the robot needs to grasp and move obstacles in order to connect disconnected components of the configuration space, with recent extensions focusing on heuristics for manipulation planning in unknown environments \cite{wu_levihn_stilman_2010,levihn_scholz_stilman_2013}. Unlike such methods that require constant deliberative replanning in the presence of unanticipated conditions, this work examines the use of a reactive vector field controller, with simultaneous guarantees of convergence and obstacle avoidance in partially known environments \cite{vasilopoulos_pavlakos_bowman_caporale_daniilidis_pappas_koditschek_2020}, endowed with a narrow symbolic interface to the abstract reactive temporal logic planner whose freedom from any consideration of geometric details affords decisive computational advantage in supervising the task.

\subsection{Reactive Temporal Logic Planning}
Reactive temporal logic planning algorithms that can account for  environmental uncertainty in terms of incomplete environment models have been developed  in \cite{guo2013revising,guo2015multi,maly2013iterative,lahijanian2016iterative,livingston2012backtracking,livingston2013patching,kress2009temporal,alonso2018reactive}. Particularly, \cite{guo2013revising,guo2015multi} model the environment as a transition system which is partially known. Then, a discrete controller is designed by applying graph search methods on a product automaton. As the environment, i.e., the transition system, is updated, the product automaton is locally updated as well, and new paths are re-designed by applying graph search approaches on the revised automaton. A conceptually similar approach is proposed in \cite{maly2013iterative,lahijanian2016iterative}. The works in \cite{livingston2012backtracking,livingston2013patching} propose methods to locally patch paths, as the transition system (modeling the environment) changes so that GR(1) (General Reactivity of Rank 1) specifications \cite{piterman2006synthesis} are satisfied. Reactive to LTL specifications planning algorithms are proposed in \cite{kress2009temporal,alonso2018reactive}, as well. Specifically, in \cite{kress2009temporal,alonso2018reactive} the robot reacts to the environment while the task specification captures this reactivity. Correctness of these algorithms is guaranteed if the robot operates in an environment that satisfies the assumptions that were explicitly modeled in the task specification. Common in all these works is that, unlike our approach, they rely on discrete abstractions of the robot dynamics  \cite{belta2005discrete,pola2008approximately} while active interaction with the environment to satisfy the logic specification is neglected.

\subsection{Contributions and Organization of the Paper}

This paper introduces the first planning and control architecture to provide a formal interface between an abstract temporal logic engine and a physically grounded mobile manipulation vector field planner for the rearrangement of movable objects in partially known workspaces cluttered with unknown obstacles. We provide conditions under which the symbolic controller is complete (Proposition~\ref{proposition:completeness}), while exploiting prior results \cite{vasilopoulos_pavlakos_bowman_caporale_daniilidis_pappas_koditschek_2020} that guarantee safe physical achievement of its sub-tasks when they are feasible, and introduce a new heuristic vector field controller for greedy physical rearrangement of the workspace when they are not. We provide a variety of simulation examples that illustrate the efficacy of the proposed algorithm for accomplishing complex manipulation tasks in unknown environments.

The paper is organized as follows. After formulating the problem in Section~\ref{sec:problem_formulation}, Section~\ref{sec:ltl_planner} presents a discrete controller which given an LTL specification generates on-the-fly high-level manipulation primitives, translated to point navigation commands through an interface layer outlined in Section~\ref{sec:logic_reactive_interface}. Using this interface, Section~\ref{sec:reactive_planner} continues with the reactive implementation of our symbolic actions and the employed algorithm for connecting disconnected freespace components blocked by movable objects. Section~\ref{sec:simulations} discusses our numerical results and, finally, Section~\ref{sec:conclusion} concludes with ideas for future work.

%% file: 2-problem-formulation.tex
\section{PROBLEM DESCRIPTION}
\label{sec:problem_formulation}
\subsection{Model of the Robot and the Environment}
We consider a first-order, nonholonomically-constrained, disk-shaped robot, centered at $\robotposition \in \reals^2$ with radius $\robotradius \in \reals_{>0}$ and orientation $\robotorientation \in S^1$; its rigid placement is denoted by $\robotpositionunicycle := (\robotposition,\robotorientation) \in \mathbb{R}^2 \times S^1$ and its input vector $\controlunicycle:=(\linearinput,\angularinput)$ consists of a fore-aft and an angular velocity command. The robot uses a gripper to move disk-shaped {\it movable objects} of known location, denoted by $\movableobjectset:=\{\movableobject_i\}_{i \in \{1,\ldots,\movableobjectcardinality\}}$ , with a vector of radii $(\movableobjectradius{1}, \ldots, \movableobjectradius{\movableobjectcardinality}) \in \reals^{\movableobjectcardinality}$, in a closed, compact, polygonal, typically non-convex workspace $\workspace \subset \reals^2$. The robot's gripper $g$ can either be engaged ($g=1$) or disengaged ($g=0$). Moreover, we adopt the perceptual model of our recent physical implementations \cite{vasilopoulos_pavlakos_bowman_caporale_daniilidis_pappas_koditschek_2020,vasilopoulos_pavlakos_schmeckpeper_daniilidis_koditschek_2019} whereby a sensor of range $\sensorrange$ recognizes and instantaneously localizes any fixed ``familiar'' or ``unfamiliar'' obstacles; see also Fig.~\ref{fig:front_figure}.

The workspace is cluttered by a finite collection of disjoint obstacles of unknown number and placement, denoted by $\obstacleset$. This set might also include non-convex ``intrusions'' of the boundary of the physical workspace $\workspace$ into the convex hull of the closure of the workspace $\workspace$, defined as the {\it enclosing workspace}. As in \cite{Arslan_Koditschek_2018,vasilopoulos_pavlakos_schmeckpeper_daniilidis_koditschek_2019,vasilopoulos_pavlakos_bowman_caporale_daniilidis_pappas_koditschek_2020}, we define the \textit{freespace} $\freespace$ as the set of collision-free placements for the closed ball $\ballclosure{\robotposition}{\robotradius}$ centered at $\robotposition$ with radius $\robotradius$, and the \textit{enclosing freespace}, $\enclosingfreespace$, as $\enclosingfreespace := \left\{ \robotposition \in \mathbb{R}^2 \, | \, \robotposition \in \text{Conv}(\overline{\freespace}) \right\}$.

Although none of the positions of any obstacles in $\obstacleset$ are \`{a}-priori known, a subset $\knownobstacleset \subseteq \obstacleset$ of these obstacles is assumed to be ``familiar'' in the sense of having a recognizable polygonal geometry, that the robot can instantly identify and localize (as we have implemented in \cite{vasilopoulos_pavlakos_bowman_caporale_daniilidis_pappas_koditschek_2020,vasilopoulos_pavlakos_schmeckpeper_daniilidis_koditschek_2019}). The remaining obstacles in $\unknownobstacleset:=\obstacleset\backslash\knownobstacleset$ are assumed to be strongly convex according to \cite[Assumption 2]{Arslan_Koditschek_2018} (and implemented in \cite{vasilopoulos_pavlakos_bowman_caporale_daniilidis_pappas_koditschek_2020,vasilopoulos_pavlakos_schmeckpeper_daniilidis_koditschek_2019}), but are otherwise completely unknown.

To simplify the notation, we dilate each obstacle and movable object by $\robotradius$ (or $\robotradius + \movableobjectradius{i}$ when the robot carries an object $i$), and assume that the robot operates in the freespace $\mathcal{F}$. We denote the set of dilated objects and obstacles derived from $\movableobjectset, \obstacleset, \knownobstacleset$ and $\unknownobstacleset$, by $\movableobjectsetdilated, \obstaclesetdilated, \knownobstaclesetdilated$ and $\unknownobstaclesetdilated$ respectively. For our formal results, we assume that each obstacle in $\unknownobstaclesetdilated$ is always well-separated from all other obstacles in both $\unknownobstaclesetdilated$ and $\knownobstaclesetdilated$, as outlined in \cite[Assumption 1]{vasilopoulos_pavlakos_bowman_caporale_daniilidis_pappas_koditschek_2020}; in practice, the surrounding environment often violates our separation assumptions, without precluding successful task completion.

\subsection{Specifying Complex Manipulation Tasks}
The robot needs to accomplish a mobile manipulation task, by visiting known regions of interest $\ell_j\subseteq\workspace$, where $j\in\{1,\dots,L\}$, for some $L>0$, and applying one of the following three manipulation actions $a_k(\movableobjectdilated_i,\ell_j)\in\ccalA$, with $\movableobjectdilated_i \in \movableobjectsetdilated$ referring to a movable object, defined as follows: 
\begin{itemize}
\item $\textsc{Move}(\ell_j)$ instructing the robot to move to region $\ell_j$, labeled as $a_1(\varnothing,\ell_j)$, where $\varnothing$ means that this action does not logically entail interaction with any specific movable object\footnote{Although, as will be detailed in Section~\ref{sec:reactive_planner}, the hybrid reactive controller may actually need to move objects out of the way, rearranging the topology of the workspace in a manner hidden from the logical task controller.}.
\item $\textsc{GraspObject}(\movableobjectdilated_i)$ instructing the robot to grasp the movable object $\movableobjectdilated_i$, labeled as $a_2(\movableobjectdilated_i, \varnothing)$, with $\varnothing$ denoting that no region is associated with this action.
\item $\textsc{ReleaseObject}(\movableobjectdilated_i,\ell_j)$ instructing the robot to push the (assumed already grasped) object $\movableobjectdilated_i$ toward its designated goal position, $\ell_j$, labeled as $a_3(\movableobjectdilated_i,\ell_j)$.
\end{itemize}

For instance, consider a rearrangement planning scenario where the locations of three objects of interest need to be rearranged, as in Fig.~\ref{fig:front_figure}. We capture such complex manipulation tasks via Linear Temporal Logic (LTL) specifications. Specifically, we use atomic predicates of the form $\pi^{a_k(\movableobjectdilated_i,\ell_j)}$, which are true when the robot applies the action $a_k(\movableobjectdilated_i,\ell_j)$ and false until the robot achieves that action. Note that these atomic predicates allow us to specify temporal logic specifications defined over manipulation primitives and, unlike related works \cite{he2015towards,shoukry2018smc}, are entirely agnostic to the geometry of the environment.
We define LTL formulas by collecting such predicates in a set $\mathcal{AP}$ of atomic propositions. For example, the rearrangement planning scenario with three movable objects initially located in regions $\ell_1$, $\ell_2$, and $\ell_3$, as shown in Fig.~\ref{fig:front_figure}, can be described as a sequencing task \cite{fainekos2005hybrid} by the following LTL formula:
\begin{align}
    \phi =& \lozenge (\pi^{a_2(\movableobjectdilated_1,\varnothing)} \wedge \lozenge (\pi^{a_3(\movableobjectdilated_1,\ell_2)} \wedge\nonumber\\ 
    &\lozenge
    (\pi^{a_2(\movableobjectdilated_2,\varnothing)} \wedge \lozenge \pi^{a_3(\movableobjectdilated_2,\ell_3)} 
    \wedge \nonumber\\ & \lozenge(\pi^{a_2(\movableobjectdilated_3,\varnothing)} \wedge \lozenge \pi^{a_3(\movableobjectdilated_3,\ell_1)})
    )))
    \label{eq:ex_task}
\end{align}
where $\lozenge$ and  $\wedge$ refer to the `eventually' and `AND' operator. In particular, this task requires the robot to perform the following steps in this order (i) grasp object $\movableobjectdilated_1$ and release it in location $\ell_2$ (first line in \eqref{eq:ex_task}); (ii) then grasp object $\movableobjectdilated_2$ and release it in location $\ell_3$ (second line in \eqref{eq:ex_task}); (iii) grasp object $\movableobjectdilated_3$ and release it in location $\ell_1$ (third line in \eqref{eq:ex_task}).
LTL formulas are satisfied over an infinite sequence of states \cite{baier2008principles}. Unlike related works where a state is defined to be the robot position, e.g., \cite{kress2009temporal}, here a state is defined by the manipulation action $a_k(\movableobjectdilated_i,\ell_j)$ that the robot applies. In other words, an LTL formula defined over manipulation-based predicates $\pi^{a_k(\movableobjectdilated_i,\ell_j)}$ is satisfied by an infinite sequence of actions $p=p_0,p_1,\dots,p_n,\dots$, where $p_n\in\ccalA$, for all $n\geq0$ \cite{baier2008principles}.
Given a sequence $p$, the syntax and semantics of LTL can be found in \cite{baier2008principles}; hereafter, we exclude the `next' operator from the syntax, since it is not meaningful for practical robotics applications \cite{kloetzer2008fully}, as well as the negation operator\footnote{Since the negation operator is excluded, safety requirements, such as obstacle avoidance, cannot be captured by the LTL formula; nevertheless, the proposed method can still handle safety constraints by construction of the (continuous-time) reactive, vector field controller in Section~\ref{sec:reactive_planner}.}.


\subsection{Problem Statement}

Given a task specification captured by an LTL formula $\phi$, our goal is to (i) generate online, as the robot discovers the environment via sensor feedback, appropriate actions using the (discrete) symbolic controller, (ii) translate them to point navigation tasks, (iii) execute these navigation tasks and apply the desired manipulation actions with a (continuous-time) vector field controller, while avoiding unknown and familiar obstacles, (iv) be able to online detect freespace disconnections that prohibit successful action completion, and (v) either locally amend the provided plan by disassembling blocking movable objects, or report failure to the symbolic controller and request an alternative action.

%% file: 3-ltl-planner-icra.tex
\section{SYMBOLIC CONTROLLER}
\label{sec:ltl_planner}
%
In this Section, we design a discrete controller that generates manipulation commands online in the form of the actions defined in Section~\ref{sec:problem_formulation} (see Fig.~\ref{fig:system}), and describe the manner in which this symbolic controller extends prior work \cite{kantaros2020reactive} to account for manipulation-based atomic predicates. A detailed construction is included in Appendix~\ref{appendix:ltl_planner}.


\subsection{Construction of the Symbolic Controller}\label{sec:nba}

First, we translate the specification $\phi$, constructed using a set of atomic predicates $\mathcal{AP}$, into a Non-deterministic B$\ddot{\text{u}}$chi Automaton (NBA) with state-space and transitions among states that can be used to measure how much progress the robot has made in terms of accomplishing the assigned mission; see Appendix~\ref{appendix:ltl_planner}.  
%
%
Particularly, we define a distance metric over this NBA state-space to compute how far the robot is from accomplishing its assigned task, or more formally, from satisfying the accepting condition of the NBA, by following a similar analysis as in \cite{kantaros2018text,kantaros2020reactive}. This metric is used to generate manipulation commands online as the robot navigates the unknown environment. The main idea is that, given its current NBA state, the robot should reach a next NBA state that decreases the distance to a state that accomplishes the assigned task. Once this target NBA state is selected, a symbolic action that achieves it is generated, in the form of a manipulation action presented in Section~\ref{sec:problem_formulation} (e.g., `release the movable object $\movableobjectdilated_i$ in region $\ell_j$'). This symbolic action acts as an input to the reactive, continuous-time controller (see Section~\ref{sec:reactive_planner}), that handles obstacles and unanticipated conditions.

When the assigned sub-task is accomplished, a new target automaton state is selected and a new manipulation command is generated. If the continuous-time controller fails to accomplish the sub-task (because e.g., a target is surrounded by fixed obstacles), the symbolic controller checks if there exists an alternative command that ensures reachability of the target automaton state (e.g., consider a case where a given NBA state can be reached if the robot goes to either region $\ell_1$ or $\ell_2$). If there are no alternative commands to reach the desired automaton state, then a new target automaton state that also decreases the distance to satisfying the accepting NBA condition is selected. If there are no such other automaton states, a message is returned stating that the robot cannot accomplish the assigned mission. A detailed description for the construction of this distance metric is provided in Appendix~\ref{appendix:ltl_planner}.

\subsection{Completeness of the Symbolic Controller}
Here, we provide conditions under which the proposed discrete controller is complete. The proof for the following Proposition can be found in Appendix~\ref{appendix:ltl_planner}.
\begin{proposition}[Completeness]\label{proposition:completeness}
Assume that there exists at least one infinite sequence of manipulation actions in the set $\ccalA$ that satisfies $\phi$. 
%
If the environmental structure and the continuous-time controller always ensure that at least one of the candidate next NBA states can be reached, then the proposed discrete algorithm is complete, i.e., a feasible solution will be found.\footnote{Given the current NBA state, denoted by $q_B(t)$, the symbolic controller selects as the next NBA state, a state that is reachable from $q_B(t)$ and closer to the final states as per the proposed distance metric; see also Appendix~\ref{appendix:ltl_planner}. All NBA states that satisfy this condition are called candidate next NBA states. Also, reaching an NBA state means that at least one of the manipulation actions required to enable the transition from $q_B(t)$ to the next NBA state is feasible.}
\end{proposition}

%% file: 4-logic-reactive-interface.tex
\section{INTERFACE LAYER BETWEEN THE SYMBOLIC AND THE REACTIVE CONTROLLER}
\label{sec:logic_reactive_interface}

We assume that the robot is nominally in an {\it LTL mode}, where it executes sequentially the commands provided by the symbolic controller described in Section~\ref{sec:ltl_planner}. We use an {\it interface layer} between the symbolic controller and the reactive motion planner, as shown in Fig.~\ref{fig:system}, to translate each action to an appropriate gripper command ($g=0$ for $\textsc{Move}$ and $\textsc{GraspObject}$, and $g=1$ for $\textsc{ReleaseObject}$), and a navigation command toward a target $\goalposition$. If the provided action is $\textsc{Move}(\ell_j)$ or $\textsc{ReleaseObject}(\movableobjectdilated_i,\ell_j)$, we pick as $\goalposition$ the centroid of region $\ell_j$. If the action is $\textsc{GraspObject}(\movableobjectdilated_i)$, we pick as $\goalposition$ a collision-free location on the boundary of object $\movableobjectdilated_i$, contained in the freespace $\freespace$.

Consider again the example shown in Fig.~\ref{fig:front_figure}. The first step of the assembly requires the robot to move object $\movableobjectdilated_1$ to $\ell_2$ which, however, is occupied by the object $\movableobjectdilated_2$. In this case, instead of reporting that the assigned LTL formula cannot be satisfied, we allow the robot to temporarily pause the command execution from the symbolic controller, switch to a {\it Fix mode} and push object $\movableobjectdilated_2$ away from $\ell_1$, before resuming the execution of the action instructed by the symbolic controller. For plan fixing purposes, we introduce a fourth action, $\textsc{DisassembleObject}(\movableobjectdilated_i,\movableobjectgoal)$, invisible to the symbolic controller, instructing the robot to push the object $\movableobjectdilated_i$ (after it has been grasped using $\textsc{GraspObject}$) toward a position $\movableobjectgoal$ on the boundary of the freespace until specific separation conditions are satisfied. Hence, an additional responsibility of the interface layer (when in Fix mode) is to pick the next object to be grasped and disassembled from a stack of blocking movable objects $\mathcal{B}_\movableobjectsetdilated$, as well as the target $\goalposition$ of each $\textsc{DisassembleObject}$ action, until the stack $\mathcal{B}_\movableobjectsetdilated$ becomes empty\footnote{The exclusion of the negation operator from the LTL syntax, as assumed in Section~\ref{sec:problem_formulation}, guarantees that each $\textsc{DisassembleObject}$ action will {\it not} interfere with the satisfaction of the LTL formula $\phi$, e.g., the robot will not disassemble an object that should not be grasped or moved.}; see Section~\ref{subsec:action_implementation}. Note that if the robot executes a $\textsc{ReleaseObject}$ action before switching to the Fix mode, the interface layer instructs it to first disassemble the carried object $\movableobjectdilated_i$ and, after disassembling all objects in $\mathcal{B}_{\movableobjectsetdilated}$, switch back to the LTL mode by re-grasping $\movableobjectdilated_i$.

Finally, the interface layer (a) requests a new action from the symbolic controller, if the robot successfully completes the current action execution, or (b) reports that the currently executed action is infeasible and requests an alternative action, if the topology checking module outlined in Section~\ref{subsec:topology_checking} determines that the target is surrounded by fixed obstacles.

%% file: 5-reactive-planner.tex
\section{SYMBOLIC ACTION IMPLEMENTATION}
\label{sec:reactive_planner}

In this Section, we describe the online implementation of our symbolic actions, assuming that the robot has already picked a target $\goalposition$ using the interface layer from Section~\ref{sec:logic_reactive_interface}. As reported above, in the LTL mode, the robot executes commands from the symbolic controller, using one of the actions $\textsc{Move}$, $\textsc{GraspObject}$ and $\textsc{ReleaseObject}$. The robot exits the LTL mode and enters the Fix mode when one or more movable objects block the target destination $\goalposition$; in this mode, it attempts to rearrange blocking movable objects using a sequence of the actions $\textsc{GraspObject}$ and $\textsc{DisassembleObject}$, before returning to the LTL mode.

The ``backbone'' of the symbolic action implementation is the reactive, vector field motion planner from prior work \cite{vasilopoulos_pavlakos_bowman_caporale_daniilidis_pappas_koditschek_2020}, allowing either a fully actuated or a differential-drive robot to provably converge to a designated fixed target while avoiding all obstacles in the environment. When the robot is gripping an object $i$, we use the method from \cite{vasilopoulos2018} for generating virtual commands for the center $\robotposition_{i,c}$ of the circumscribed disk with radius $(\movableobjectradius{i} + \robotradius)$, enclosing the robot and the object. Namely, we assume that the robot-object pair is a fully actuated particle with dynamics $\dot{\robotposition}_{i,c} = \controlfullyactuated_{i,c}(\robotposition_{i,c})$, design our control policy $\controlfullyactuated_{i,c}$ using the same vector field controller, and translate to commands $\controlunicycle = (\linearinput,\angularinput)$ for our differential drive robot as $\controlunicycle := \mathbf{T}_{i,c}(\robotorientation)^{-1} \, \controlfullyactuated_{i,c}$, with $\mathbf{T}_{i,c}(\robotorientation)$ the Jacobian of the gripping contact, i.e., $\dot{\robotposition}_{i,c} = \mathbf{T}_{i,c}(\robotorientation) \, \controlunicycle$.

This reactive controller assumes that a path to the goal always exists (i.e., the robot's freespace is path-connected), and does not consider cases where the target is blocked either by a fixed obstacle or a movable object\footnote{The possibility of an entirely unknown blocking convex obstacle is precluded by our separation assumptions in Section~\ref{sec:problem_formulation}.}. Hence, here we extend the algorithm's capabilities by providing a topology checking algorithm (Section~\ref{subsec:topology_checking}) that detects blocking movable objects or fixed obstacles, as outlined in Fig.~\ref{fig:system}. Based on these capabilities, Section~\ref{subsec:action_implementation} describes our symbolic action implementations. Appendix~\ref{appendix:reactive_controller_overview} includes a brief overview of the reactive, vector field motion planner, and Appendix~\ref{appendix:topology_checking_algorithm} includes an algorithmic outline of our topology checking algorithm.

\subsection{Topology Checking Algorithm}
\label{subsec:topology_checking}

The topology checking algorithm is used to detect freespace disconnections, update the robot's enclosing freespace $\enclosingfreespace$, and modify its action by switching to the Fix mode, if necessary. In summary, the algorithm's input is the initially assumed polygonal enclosing freespace $\enclosingfreespace$ for either the robot or the robot-object pair, along with all known dilated movable objects in $\movableobjectsetdilated$ and fixed obstacles in $\knownobstaclesetdilated_\hybridmode$ (corresponding to the index set $\hybridmode$ of localized familiar obstacles). The algorithm's output is the detected enclosing freespace $\enclosingfreespace$, used for the diffeomorphism construction in the reactive controller \cite{vasilopoulos_pavlakos_bowman_caporale_daniilidis_pappas_koditschek_2020}, along with a stack of {\it blocking movable objects} $\mathcal{B}_{\movableobjectsetdilated}$ and a Boolean indication of whether the current symbolic action is feasible. Based on this output, the robot switches to the Fix mode when the stack $\mathcal{B}_{\movableobjectsetdilated}$ becomes non-empty, and resumes execution from the symbolic controller once all movable objects in $\mathcal{B}_{\movableobjectsetdilated}$ are disassembled. A detailed description is given in Appendix~\ref{appendix:topology_checking_algorithm}.

\subsection{Action Implementation}
\label{subsec:action_implementation}

We are now ready to describe the used symbolic actions. The symbolic action $\textsc{Move}(\ell_j)$ simply uses the reactive controller to navigate to the selected target $\goalposition$, as described in Appendix~\ref{appendix:reactive_controller_overview}. Similarly, the symbolic action $\textsc{GraspObject}(\movableobjectdilated_i)$ uses the reactive controller to navigate to a collision-free location on the boundary of object $\movableobjectdilated_i$, and then aligns the robot so that its gripper faces the object, in order to get around Brockett's condition \cite{Brockett-83}. $\textsc{ReleaseObject}(\movableobjectdilated_i,\ell_j)$ uses the reactive controller to design inputs for the robot-object center $\robotposition_{i,c}$ and translates them to differential drive commands through the center's Jacobian $\mathbf{T}_{i,c}(\robotorientation)$, in order to converge to the goal $\goalposition$.

Finally, the action $\textsc{DisassembleObject}(\movableobjectdilated_i,\goalposition)$ is identical to $\textsc{ReleaseObject}$, with two important differences. First, we heuristically select as target $\goalposition$ the middle point of the edge of the polygonal freespace $\freespace$ that maximizes the distance to all other movable objects (except $\movableobjectdilated_i$) and all regions of interest $\ell_j$. Second, in order to accelerate performance and shorten the resulting trajectories, we stop the action's execution if the robot-object pair, centered at $\robotposition_{i,c}$ does not intersect any region of interest and the distance of $\robotposition_{i,c}$ from all other objects in the workspace is at least $2(\robotradius + \max_{k \in \mathcal{B}_{\movableobjectsetdilated}} \movableobjectradius{k})$, as this would imply that dropping the object in its current location would not block a next step of the disassembly process. Even though we do not yet report on formal results pertaining to the task sequence in the Fix mode, the $\textsc{DisassembleObject}$ action maintains formal results of obstacle avoidance and target convergence to a feasible $\goalposition$, using our reactive, vector field controller.

%% file: 6-simulations.tex
\section{ILLUSTRATIVE SIMULATIONS}
\label{sec:simulations}

\begin{figure}
\captionsetup{width=\linewidth,font=footnotesize}
\centering
\includegraphics[width=1.0\columnwidth]{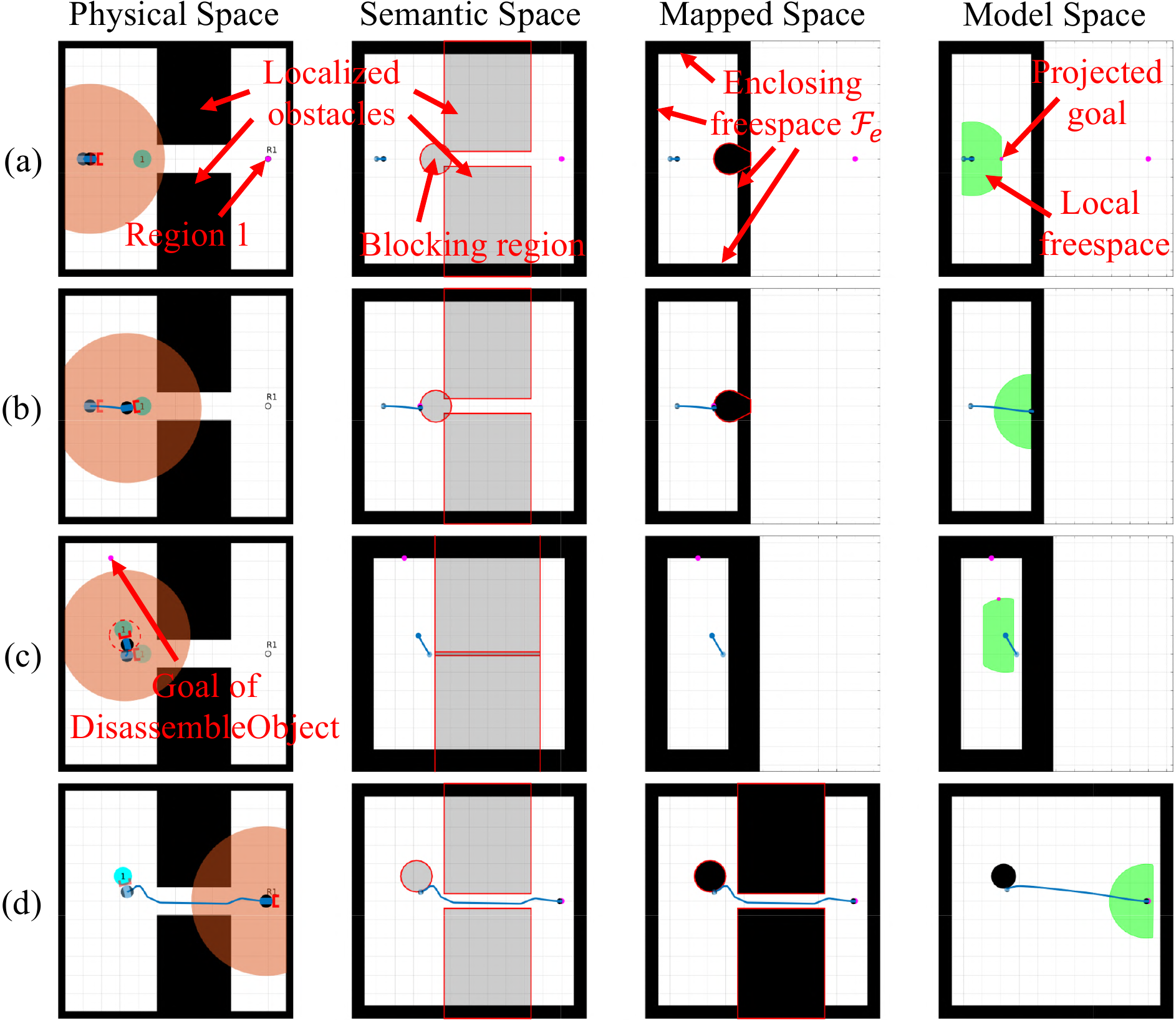}
\caption{Demonstration of local LTL plan fixing, where the task is to navigate to region 1, captured by the LTL formula $\phi = \lozenge \pi^{a_1(\varnothing,\ell_1)}$ where $\ell_1$ refers to region 1 in the figure. (a) The robot starts navigating to its target, until it localizes the two rectangular obstacles and recognizes that the only path to the goal is blocked by a movable object. (b) The robot switches to the Fix mode, grips the object, and (c) moves it away from the blocking region, until the separation assumptions outlined in Section~\ref{subsec:action_implementation} are satisfied. (d) It then proceeds to complete the task.}
\label{fig:simulation_simple}
\end{figure}

\begin{figure}[t]
\captionsetup{width=\linewidth,font=footnotesize}
\centering
\includegraphics[width=1.0\columnwidth]{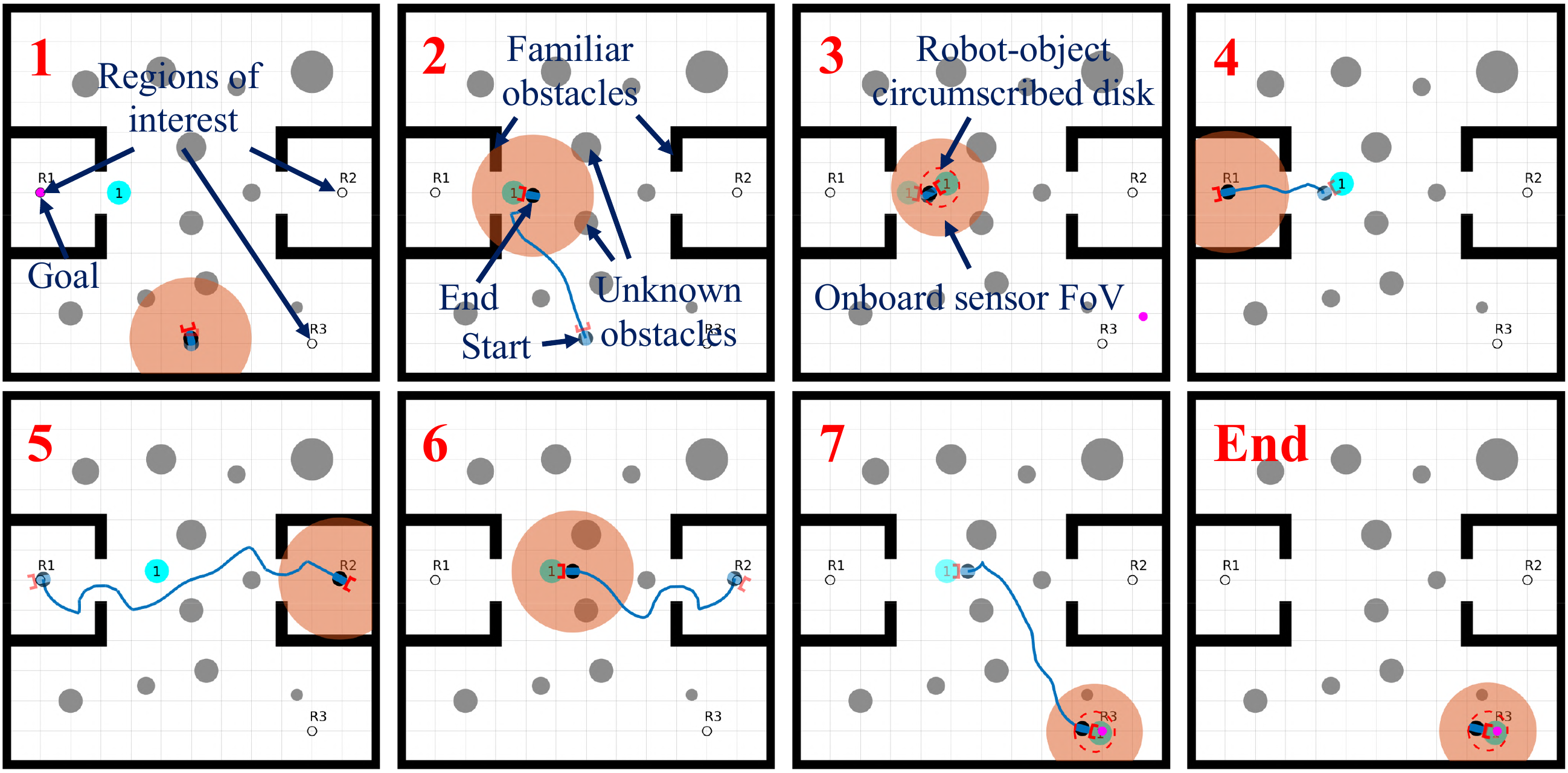}
\caption{Executing the LTL formula $\phi = \lozenge (\pi^{a_1(\varnothing,\ell_1)} \wedge \lozenge (\pi^{a_1(\varnothing,\ell_2)} \wedge \lozenge (\pi^{a_2(\movableobjectdilated_1,\varnothing)} \wedge \lozenge \pi^{a_3(\movableobjectdilated_1,\ell_3)})))$ in an environment cluttered with known walls (black) and unknown convex obstacles (grey).}
\label{fig:simulation_sequential}
\end{figure}

\begin{figure}
\captionsetup{width=\linewidth,font=footnotesize}
\centering
\includegraphics[width=0.8\columnwidth]{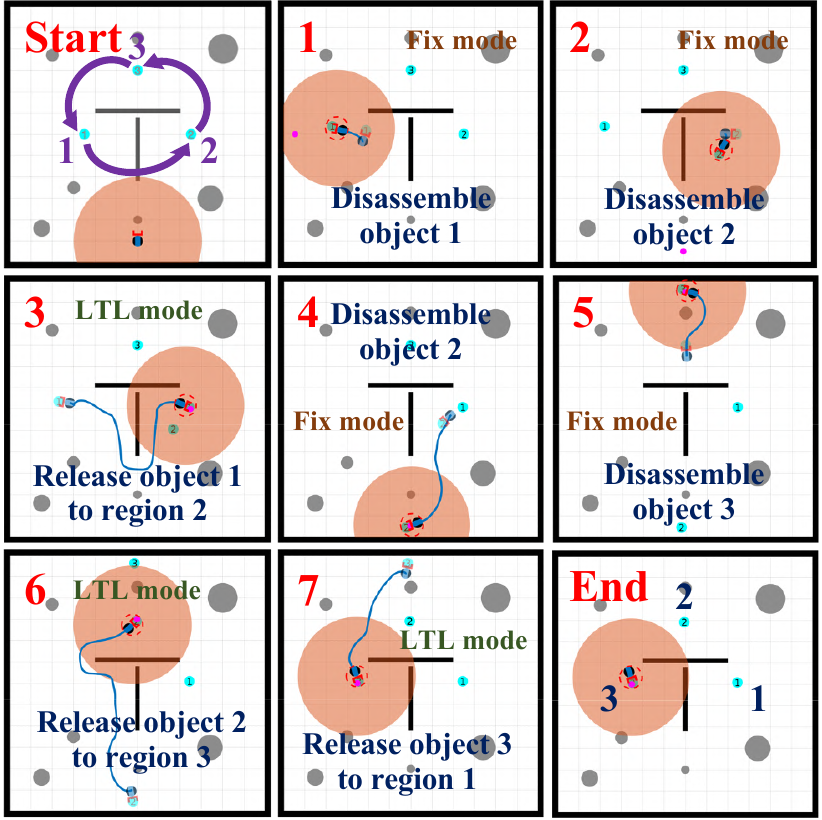}
\caption{An illustrative execution of the problem depicted in Fig.~\ref{fig:front_figure}. The task, specified by the LTL formula \eqref{eq:ex_task}, requires the counterclockwise rearrangement of 3 objects in an environment cluttered with some unanticipated familiar (initially dark grey and then black upon localization) and some completely unknown (light grey) fixed obstacles.}
\label{fig:simulation_rearrangement}
\vspace{-14pt}
\end{figure}

In this Section, we implement simulated examples of different tasks in various environments using our architecture, shown in Fig.~\ref{fig:system}. All simulations were run in MATLAB using $\texttt{ode45}$, leveraging and enhancing existing presentation infrastructure\footnote{See \url{https://github.com/KodlabPenn/semnav_matlab}.}. The discrete controller and the interface layer are implemented in MATLAB, whereas the reactive controller is implemented in Python and communicates with MATLAB using the standard MATLAB-Python interface. For our numerical results, we assume perfect robot state estimation and localization of obstacles using the onboard sensor, which can instantly identify and localize either the entirety of familiar obstacles or fragments of unknown obstacles within its range\footnote{The reader is referred to \cite{vasilopoulos_pavlakos_bowman_caporale_daniilidis_pappas_koditschek_2020,vasilopoulos_pavlakos_schmeckpeper_daniilidis_koditschek_2019} for examples of physical realization of such sensory assumptions, using a combination of an onboard camera for obstacle recognition and pose estimation, and a LIDAR sensor for extracting distance to nearby obstacles.}. The reader is referred to the accompanying video submission for visual context and additional simulations.

\subsection{Demonstration of Local LTL Plan Fixing}
Fig.~\ref{fig:simulation_simple} includes a demonstration of a simple task, encoded in the LTL formula $\phi = \lozenge \pi^{a_1(\varnothing,\ell_1)}$, i.e., eventually execute the action $\textsc{Move}$ to navigate to region 1, demonstrating how the Fix mode for local rearrangement of blocking movable objects works.

\subsection{Executing More Complex LTL Tasks}
Fig.~\ref{fig:simulation_sequential} includes successive snapshots of a more complicated LTL task, captured by the formula
\begin{equation*}
    \phi = \lozenge (\pi^{a_1(\varnothing,\ell_1)} \wedge \lozenge (\pi^{a_1(\varnothing,\ell_2)} \wedge \lozenge (\pi^{a_2(\movableobjectdilated_1,\varnothing)} \wedge \lozenge \pi^{a_3(\movableobjectdilated_1,\ell_3)})))
\end{equation*}
which instructs the robot to first navigate to region 1, then navigate to region 2, and finally grasp object 1 and move it to region 3, in an environment cluttered with both familiar non-convex and completely unknown convex obstacles. Before navigating to region 1, the robot correctly identifies that the movable object disconnects its freespace and proceeds to disassemble it. After visiting region 2, it then revisits the movable object, grasps it and moves it to the designated location to complete the required task. The reader is referred to the video submission for visual context regarding the evolution of all planning spaces \cite{vasilopoulos_pavlakos_bowman_caporale_daniilidis_pappas_koditschek_2020} (semantic, mapped and model space) during the execution of this task, as well as several other simulations with more movable objects, including (among others) a task where the robot needs to patrol between some predefined regions of interest in an environment cluttered with obstacles by visiting each one of them infinitely often.

\subsection{Execution of Rearrangement Tasks}
Finally, a promising application of our reactive architecture concerns rearrangement planning with multiple movable pieces. Traditionally, such tasks are executed using sampling-based planners, whose offline search times can blow up exponentially with the number of movable pieces in the environment (see, e.g., \cite[Table I]{vegabrown_rss2017}). Instead, as shown in Fig.~\ref{fig:simulation_rearrangement}, the persistent nature of our reactive architecture succeeds in achieving the given task online in an environment with multiple obstacles, even though our approach might require more steps and longer trajectories in the overall assembly process than other optimal algorithms \cite{vegabrown_2018}. Moreover, the LTL formulas for encoding such tasks are quite simple to write (see \eqref{eq:ex_task} for the example in Fig.~\ref{fig:simulation_rearrangement}), instructing the robot to grasp and release each object in sequence; the reactive controller is capable of handling obstacles and blocking objects during execution time. Our video submission includes a rearrangement example with 4 movable objects, requiring more steps in the assembly process.

%% file: 7-conclusion.tex
\section{CONCLUSION}
\label{sec:conclusion}

In this paper, we propose a novel hybrid control architecture for achieving complex tasks with mobile manipulators in the presence of unanticipated obstacles and conditions. Future work will focus on providing end-to-end (instead of component-wise) correctness guarantees, extensions to multiple robots for collaborative manipulation tasks, as well as physical experimental validation.

%% file: appendix_LTL.tex
\section{Detailed Description of the Symbolic Controller}
\label{appendix:ltl_planner}
%
This Appendix provides a detailed description of the distance metric used in Section~\ref{sec:ltl_planner} upon which a discrete controller is designed that generates manipulation commands online. To accomplish this, first in Appendix~\ref{sec:nbaAp} we translate the LTL formula into a Non-deterministic B$\ddot{\text{u}}$chi Automaton (NBA) and we provide a formal definition of its accepting condition. 
Then, in Appendix~\ref{sec:dist}, we provide a detailed description for the construction of the distance metric over this automaton state space. Appendix~\ref{appendix:planning} describes our method for generating symbolic actions online, and Appendix~\ref{appendix:completeness} includes the proof of our completeness result. The proposed method is also illustrated in Figs.~\ref{fig:NBA}-\ref{fig:graphG}.


\subsection{From LTL Formulas to Automata}\label{sec:nbaAp}

\begin{figure}[t]
  \centering
  \captionsetup{width=\linewidth,font=footnotesize}
  \includegraphics[width=0.6\linewidth]{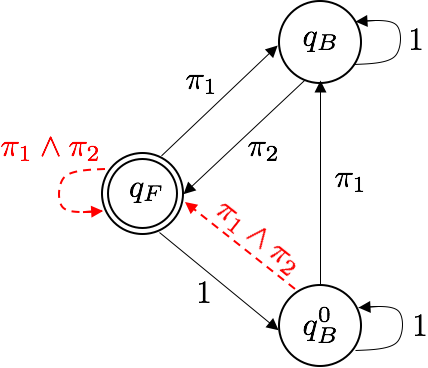}
  \caption{Graphical illustration of the NBA corresponding to the LTL formula $\phi=\square\Diamond(\pi_1)\wedge\square\Diamond(\pi_2)$ where for simplicity of notation $\pi_1=\pi^{a_1(\varnothing, \ell_1)}$ and $\pi_2=\pi^{a_1(\varnothing, \ell_2)}$. The automaton has been generated using the tool in \cite{gastin2001fast}. In words, this LTL formula requires the robot to visit infinitely often and in any order the regions $\ell_1$ and $\ell_2$. The initial state of the automaton is denoted by $q_B^0$ while the final state is denoted by $q_F$. When the robot is in an NBA state and the Boolean formula associated with an outgoing transition from this NBA state is satisfied, then this transition can be enabled. For instance, when the robot is in the initial state $q_B^0$ and satisfies the atomic predicate $\pi_1$, the transition from $q_B^0$ to $q_B$ can be enabled, i.e., $q_B\in\delta_B(q_B^0,\pi_1)$. The LTL formula is satisfied if starting from $q_B^0$, the robot generates an infinite sequence of observations (i.e., atomic predicates that become true) that yields an infinite sequence of transitions so that the final state $q_F$ is visited infinitely often. The red dashed lines correspond to infeasible NBA transitions as they are enabled only if the Boolean formula $\pi_1\wedge\pi_2$ is satisfied, i.e., only if the robot is in more than one region simultaneously; such edges are removed yielding the pruned NBA. }
  \label{fig:NBA}
\end{figure}

As reported in Section~\ref{sec:ltl_planner}, we first translate the specification $\phi$, constructed using a set of atomic predicates $\mathcal{AP}$, into a Non-deterministic B$\ddot{\text{u}}$chi Automaton (NBA) defined as follows; see also Fig.~\ref{fig:NBA}. 
 \begin{definition}[NBA]
 A Non-deterministic B$\ddot{\text{u}}$chi Automaton (NBA) $B$ over $\Sigma=2^{\mathcal{AP}}$ is defined as a tuple $B=\left(\ccalQ_{B}, \ccalQ_{B}^0,\delta_B, \ccalQ_F\right)$, where (i) $\ccalQ_{B}$ is the set of states;
 (ii) $\ccalQ_{B}^0\subseteq\ccalQ_{B}$ is a set of initial states; (iii) $\delta_B:\ccalQ_B\times\Sigma\rightarrow 2^{\ccalQ_B}$ is a non-deterministic transition relation, and $\ccalQ_F\subseteq\ccalQ_{B}$ is a set of accepting/final states.
 \label{def:nba}
 \end{definition}
To interpret a temporal logic formula over the trajectories of the robot system, we use a labeling function $L:\ccalA\rightarrow 2^{\mathcal{AP}}$ that determines which atomic propositions are true given the current robot action $a_k(\movableobjectdilated_i,\ell_j)$; note that, by definition, these actions also encapsulate the position of the robot in the environment.
%
%
%
An infinite sequence $p=p(0)p(1)\dots p(k)\dots$ of actions $p(k)\in\ccalA$, satisfies $\phi$ if the word $\sigma=L(p(0))L(p(1))\dots$ yields an accepting NBA run defined as follows \cite{baier2008principles}. First, a run $\rho_B$ of $B$ over an infinite word $\sigma=\sigma(1)\sigma(2)\dots\sigma(k)\dots\in(2^{\mathcal{AP}})^{\omega}$, is a sequence $\rho_B=q_B^0q_B^1q_B^2\dots,q_B^k,\dots$, where $q_B^0\in\ccalQ_B^0$ and $q_B^{k+1}\in\delta_B(q_B^k,\sigma(k))$, $\forall k\in\mathbb{N}$. A run $\rho_B$ is called \textit{accepting} if at least one final state appears infinitely often in it. In words, an infinite-length discrete plan $\tau$ satisfies an LTL formula $\phi$ if it can generate at least one accepting NBA run.

\subsection{Distance Metric Over the NBA}\label{sec:dist}

\begin{figure}[t]
  \centering
  \captionsetup{width=\linewidth,font=footnotesize}
  \includegraphics[width=1\linewidth]{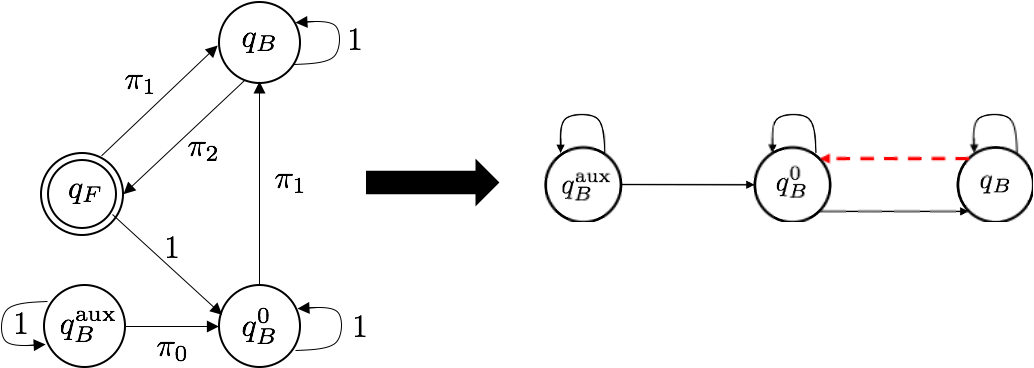}
  \caption{Graphical illustration of the graph $\ccalG$ construction for the NBA shown in Fig.~\ref{fig:NBA}. The left figure corresponds to the pruned automaton after augmenting its state space with the state $q_B^{\text{aux}}$, where $\pi_0$ corresponds to the atomic predicate that the robot satisfies initially at $t=0$. If no atomic predicates are satisfied initially, then $\pi_0$ corresponds to the empty symbol \cite{baier2008principles}. Observe in the left figure that $\ccalD_{q_B^{\text{aux}}}=\{q_B^{\text{aux}},q_B^0,q_B\}$. The right figure illustrates  the graph $\ccalG$ corresponding to this automaton. The red dashed line corresponds to an accepting edge. Also, we have that $\ccalV_F=\{q_B\}$, $d_F(q_B^{\text{aux}},\ccalV_F)=2$, $d_F(q_B^{0},\ccalV_F)=1$, and $d_F(q_B,\ccalV_F)=0$. For instance, every time the robot reaches the state $q_B^0$ with $d_F(q_B^{0},\ccalV_F)=1$, it generates a symbol to reach the state $q_B$ since reaching this state decreases the distance to the set of accepting edges (since $d_F(q_B,\ccalV_F)=0$). The symbol that can enable this transition is the symbol that satisfies the Boolean formula $b^{q_B^0,q_B}=\pi_1$; this formula is trivially satisfied by the symbol $\pi_1=\pi^{a_1(\varnothing, \ell_1)}$. As a result the command send to the continuous time controller is `Move to Region $\ell_1$'. 
  }
  \label{fig:graphG}
\end{figure}

In this Section, given a graph constructed using the NBA, we define a function to compute how far an NBA state is from the set of final states.
Following a similar analysis as in \cite{kantaros2018text,kantaros2020reactive}, we first prune the NBA by removing infeasible transitions that can never be enabled as they require the robot to be in more than one region and/or take more that one action simultaneously.  
Specifically, a symbol $\sigma\in\Sigma:=2^{\mathcal{AP}}$ is \textit{feasible} if and only if $\sigma\not\models b^{\text{inf}}$, where $b^{\text{inf}}$ is a Boolean formula defined as 
\begin{align}\label{bInfTSi}
b^{\text{inf}}=&[(\vee_{\forall k,r,j,e\neq j}(\pi^{a_k(\cdot,\ell_e)}\wedge\pi^{a_r(\cdot,\ell_j)}))]\nonumber\\
&\bigvee [(\vee_{\forall j, k,r\neq k}(\pi^{a_k(\cdot,\ell_j)}\wedge\pi^{a_r(\cdot,\ell_j)}))]
\end{align}
In words, $b^{\text{inf}}$ requires the robot to be either present simultaneously in more than one region \textit{or} take more than one action in a given region at the same time. Specifically, the first line requires the robot to be present in locations $\ell_j$ and $\ell_e$, $e\neq j$ and apply the actions $a_k,a_r\in\ccalA$ while the second line requires the robot to take two distinct actions $a_k(\cdot,\ell_j)$ and $a_r(\cdot,\ell_j)$ at the same region $\ell_j$, simultaneously.

Next, we define the sets that collect all feasible symbols that enable a transition from an NBA state $q_B$ to another, not necessarily different, NBA state $q_B'$. This definition relies on the fact that transition from a state $q_B$ to a state $q_B'$ is enabled if a Boolean formula, denoted by $b^{q_B,q_B'}$ and defined over the set of atomic predicates $\mathcal{AP}$, is satisfied. In other words, $q_B'\in\delta_B(q_B,\sigma)$, i.e., $q_B'$ can be reached from the NBA state $q_B$ under the symbol $\sigma$, if $\sigma$  satisfies $b^{q_B,q_B'}$. An NBA transition from $q_B$ to $q_B'$ is infeasible if there are no feasible symbols that satisfy $b^{q_B,q_B'}$. All infeasible NBA transitions are removed yielding a pruned NBA automaton. All feasible symbols that satisfy $b^{q_B,q_B'}$ are collected in the set $\Sigma^{q_B,q_B'}$.

To take into account the initial robot state in the construction of the distance metric, in the pruned automaton we introduce an auxiliary state $q_B^{\text{aux}}$ and transitions from $q_B^{\text{aux}}$ to all initial states $q_B^0\in\ccalQ_B^0$ so that $b^{q_B^{\text{aux}},q_B^{\text{aux}}}=1$ and $b^{q_B^{\text{aux}},q_B^{0}}=\pi_0$, i.e., transition from $q_B^\text{aux}$ to $q_B^0$ can always be enabled based on the atomic predicate that is initially satisfied denoted by $\pi_0$; note that if no predicates are satisfied initially, then $\pi_0$ corresponds to the empty symbol \cite{baier2008principles}. Hereafter, the auxiliary state $q_B^{\text{aux}}$ is considered to be the initial state of the resulting NBA; see also Fig.~\ref{fig:graphG}. 

Next, we collect all NBA states that can be reached from $q_B^{\text{aux}}$ in a possibly multi-hop fashion, using a finite sequence of feasible symbols, so that once these states are reached, the robot can always remain in them as long as needed using the same symbol that allowed it to reach this state. Formally, let $\ccalD_{q_B^{\text{aux}}}$ be a set that collects all NBA states $q_B$ (i) that have a feasible self-loop, i.e., $\Sigma^{q_B,q_B}\neq\emptyset$ and (ii) for which there exists a finite and feasible word $w$, i.e., a finite sequence of feasible symbols, so that starting from $q_B^{\text{aux}}$ a finite NBA run $\rho_w$ (i.e., a finite sequence of NBA states) is incurred that ends in $q_B$ and activates the self-loop of $q_B$. In math, $\ccalD_{q_B^{\text{aux}}}$ is defined as:
\begin{align}\label{eq:reach}
    \ccalD_{q_B^{\text{aux}}}=&\{q_B\in\ccalQ_B|\\&(\Sigma^{q_B,q_B}\neq\emptyset)\wedge(\exists w~\text{s.t.}~\rho_{w}=q_B^{\text{aux}}\dots \bar{q}_B q_Bq_B)\nonumber\}.
\end{align}
By definition of $q_B^{\text{aux}}$, we have that $q_B^{\text{aux}}\in\ccalD_{q_B^{\text{aux}}}$.

Among all possible pairs of states in $\ccalD_{q_B^{\text{aux}}}$, we examine which transitions, possibly multi-hop, can be enabled using feasible symbols, so that, once these states are reached, the robot can always remain in them forever using the same symbol that allowed it to reach this state. Formally, consider any two states $q_B, q_B'\in\ccalD_{q_B^{\text{aux}}}$ (i) that are connected through a - possibly multi-hop - path in the NBA, and (ii) for which there exists a symbol, denoted by $\sigma^{q_B,q_B'}$, so that if it is repeated a finite number of times starting from $q_B$, the following finite run can be generated: 
\begin{equation}\label{eq:run}
    \rho=q_B q_B^1\dots q_B^{K-1}q_B^{K}q_B^{K},
\end{equation}
where $q_B'=q_B^K$, for some finite $K>0$. In \eqref{eq:run}, the run is defined so that (i) $q_B^k\neq q_B^{k+1}$, for all $k\in\{1,K-1\}$; (ii) $q_B^{k}\in\delta_B(q_B^k,\sigma^{q_B,q_B'})$ is not valid for all $\forall k\in\{1,\dots,K-1\}$, i.e., the robot cannot remain in any of the intermediate states (if any) that connect $q_B$ to $q_B'$ either because a feasible self-loop does not exist or because $\sigma^{q_B,q_B'}$ cannot activate this self-loop; and (iii) $q_B'\in\delta_B(q_B',\sigma^{q_B,q_B'})$ i.e., there is a feasible loop associated with $q_B'$ that is activated by $\sigma^{q_B,q_B'}$. Due to (iii), the robot can remain in $q_B'$ as long as $\sigma^{q_B,q_B'}$ is generated. The fact that the finite repetition of a \textit{single} symbol needs to generate the run \eqref{eq:run} precludes multi-hop transitions from $q_B$ to $q_B'$ that require the robot to jump from one region of interest to another one instantaneously as such transitions are not meaningful as discussed in Section~\ref{sec:problem_formulation}; see also Fig.~\ref{fig:graphG}. Hereafter, we denote the - potentially multi-hop - transition incurred due to the run \eqref{eq:run} by  $q_B'\in\delta_{B}^m(q_B,\cdot)$.

%
Then, we construct the directed graph $\ccalG=\{\ccalV,\ccalE\}$ where $\ccalV\subseteq\ccalQ_B$ is the set of nodes and $\ccalE\subseteq\ccalV\times\ccalV$ is the set of edges. The set of nodes is defined so that $\ccalV=\ccalD_{q_B^{\text{aux}}}$ and the set of edges is defined so that $(q_B,q_B')\in\ccalE$ if there exists a feasible symbol that incurs the run $\rho_w$ defined in \eqref{eq:run}; see also Fig.~\ref{fig:graphG}. 

Given the graph $\ccalG$, we define the following distance metric.

\begin{definition}[Distance Metric]
Let $\ccalG=\{\ccalV,\ccalE\}$ be the directed graph that corresponds to NBA $B$. Then, we define the distance function $d: \ccalV \times \ccalV \rightarrow \mathbb{N}$ as follows
\begin{equation}\label{eq:dist}
d(q_B,q_B')=\left\{
                \begin{array}{ll}
                  |SP_{q_B,q_B'}|, \mbox{if $SP_{q_B,q_B'}$ exists,}\\
                  \infty, ~~~~~~~~~\mbox{otherwise},
                \end{array}
              \right.
\end{equation}
where $SP_{q_B,q_B'}$ denotes the shortest path (in terms of hops) in $\ccalG$ from $q_B$ to $q_B'$ and $|SP_{q_B,q_B'}|$ stands for its cost (number of hops). 
\end{definition}

In words, $d:\ccalV\times\ccalV\rightarrow \mathbb{N}$ returns the minimum number of edges in the graph $\ccalG$ that are required to reach a state $q_B'\in\ccalV$ starting from a state $q_B\in\ccalV$. This metric can be computed using available shortest path algorithms, such the Dijkstra method with worst-case complexity $O(|\ccalE| + |\ccalV|\log|\ccalV|)$. 
%

%
Next, we define the final/accepting edges in $\ccalG$ as follows.
\begin{definition}[Final/Accepting Edges]\label{def:accEdges}
An edge $(q_B,q_B')\in\ccalE$ is called final or accepting if the corresponding multi-hop NBA transition $q_B'\in\delta_B^m(q_B,\cdot)$ includes at least one final state $q_F\in\ccalQ_F$.
\end{definition}

Based on the definition of accepting edges, we define the set $\ccalV_F\subseteq \ccalV$ that collects all states $q_B\in\ccalV$ from which an accepting edge originates, i.e.,
\begin{equation}\label{eq:accNode}
    \ccalV_F = \{q_B\in\ccalV~|~\exists~ \text{accepting edge}~(q_B,q_B')\in\ccalE \}.
\end{equation}
By definition of the accepting condition of the NBA, we have that if at least one of the accepting edges is traversed infinitely often, then the corresponding LTL formula is satisfied. 

Similar to \cite{bisoffi2018hybrid}, we define the distance of any state $q_B\in\ccalV$ to the set $\ccalV_F\subseteq \ccalV$ as
\begin{equation}\label{eq:distF}
d_F(q_B,\ccalV_F)=\min_{q_B'\in\ccalV_F}d(q_B,q_B'),
\end{equation}
where $d(q_B,q_B')$ is defined in \eqref{eq:dist} and $\ccalV_F$ is defined in \eqref{eq:accNode}; see also Fig.~\ref{fig:graphG}.

\subsection{Online Symbolic Controller}
\label{appendix:planning}

In this Section, we present how manipulation commands are generated online. The proposed controller requires as an input the graph $\ccalG$ defined in Appendix~\ref{sec:dist}. The main idea is that as the robot navigates the unknown environment, it selects NBA states that it should visit next so that the distance to the final states, as per \eqref{eq:distF}, decreases over time. 

Let $q_B(t)\in\ccalV$ be the NBA state that the robot has reached after navigating the unknown environment for $t$ time units. At time $t=0$, $q_B(t)$ is selected to be the initial NBA state. Given the current NBA state $q_B(t)$, the robot selects a new NBA state, denoted by $q_B^{\text{next}}\in\ccalV$ that it should reach next to make progress towards accomplishing their task. This state is selected among the neighbors of $q_B(t)$ in the graph $\ccalG$ based on the following two cases.
If $q_B(t)\notin \ccalV_F$, where $\ccalV_F$ is defined in \eqref{eq:accNode}, then among all neighboring nodes, we select one that satisfies
\begin{equation}\label{eq:minDist}
d_F(q_B^{\text{next}},\ccalV_F) =  d_F(q_B(t),\ccalV_F)-1,  
\end{equation}
i.e., a state that is one hop closer to the set $\ccalV_F$ than $q_B(t)$ is where $d_F$ is defined in \eqref{eq:distF}.  Under this policy of selecting $q_B^{\text{next}}$, we have that eventually $q_B(t)\in\ccalV_F$; controlling the robot to ensure this property is discussed in Section~\ref{sec:reactive_planner}. If $q_B(t)\in\ccalV_F$, then the state $q_B^{\text{next}}$ is selected so that $(q_B(t),q_B^{\text{next}})$ is an accepting edge as per Definition~\ref{def:accEdges}. This way we ensure that accepting edges are traversed infinitely often and, therefore, the assigned LTL task is satisfied.

Given the selected state $q_B^{\text{next}}$, a feasible symbol is selected that can enable the transition from $q_B(t)$ to $q_B^{\text{next}}$, i.e., can incur the run \eqref{eq:run}. By definition of the run in \eqref{eq:run}, it suffices to select a symbol that satisfies the following Boolean formula:
\begin{equation}\label{eq:b}
   b^{q_B,q_B'}=b^{q_B,q_B^1}\wedge b^{q_B^2,q_B^3}\wedge\dots b^{q_B^{K-1},q_B^{K}}\wedge b^{q_B^{K},q_B^{K}},
\end{equation}
where $q_B^{K} = q_B^{\text{next}}$. In words, the Boolean formula in \eqref{eq:b} is the conjunction of all Boolean formulas $b^{q_B^{k-1},q_B^{k}}$ that need to be satisfied simultaneously to reach $q_B^{\text{next}}=q_B^K$ through a multi-hop path. Once such a symbol is generated, a point-to-point navigation and manipulation command is accordingly generated. For instance, if this symbol is $\pi^{a_k(\movableobjectdilated_i,\ell_j)}$ then the robot has to
apply the action $a_k(\movableobjectdilated_i,\ell_j)$, i.e.,
go to a known region of interest $\ell_j$ and apply action $a_k$ to the movable object $\movableobjectdilated_i$. The online implementation of such action is discussed in Section~\ref{sec:reactive_planner}. 

\subsection{Completeness of the Symbolic Controller}
\label{appendix:completeness}
In what follows, we provide the proof of Proposition~\ref{proposition:completeness}.

\begin{proof}[Proof of Proposition~\ref{proposition:completeness}]
To show this result it suffices to show that eventually the accepting condition of the NBA is satisfied, i.e., the robot will visit at least one of the final NBA states $q_F$ infinitely often. Equivalently, as discussed in Appendix~\ref{sec:dist}, it suffices to show that accepting edges $(q_B,q_B')\in\ccalE$, where $q_B,q_B'\in\ccalV$ are traversed infinitely often. 

First, consider an infinite sequence of time instants $\bbt=t_0,t_1,\dots,t_k,\dots$ where $t_{k+1}\geq t_k$, so that an edge in $\ccalG$, defined in Appendix~\ref{sec:dist}, is traversed at every time instant $t_k$. Let $e(t_k)\in\ccalE$ denote the edge that is traversed at time $t_k$. Thus, $\bbt$ yields the following sequence of edges $\bbe=e(t_0),e(t_1),\dots,e(t_k)\dots$ where $e(t_k)=(q_B(t_{k}),q_B(t_{k+1}))$, $q_B(t_0)=q_B^{\text{aux}}$, $q_B(t_k)\in\ccalV$, and the state $q_B^{k+1}$ is defined based on the following two cases. 
If $q_B(t_k)\notin \ccalV_F$, then the state $q_B(t_{k+1})$ is closer to $\ccalV_F$ than $q_B(t_k)$ is, i.e., $d_F(q_B(t_{k+1}),\ccalV_F)=d_F(q_B(t_{k}),\ccalV_F)-1$, where $d_F$ is defined in \eqref{eq:distF}. If $q_B(t_{k})\in\ccalV_F$, then $q_B(t_{k+1})$  is selected so that an accepting edge originating from $q_B(t_{k})$ is traversed. By definition of $q_B(t_{k})$, the `distance' to $\ccalV_F$ decreases as $t_k$ increases, i.e., given any time instant $t_k$, there exists a time instant $t_k'\geq t_k$ so that $q_B(t_{k}')\in\ccalV_F$ and then at the next time instant an accepting edge is traversed. This means that $\bbe$ includes an infinite number of accepting edges. This sequence $\bbe$ exists since, by assumption, there exists an infinite sequence of manipulation actions that satisfies $\phi$. Particularly, recall that by construction of the graph $\ccalG$, the set of edges in this graph captures all NBA transitions besides those that (i) require the robot to be in more than one region simultaneously or (ii) multi-hop NBA transitions that require the robot to jump instantaneously from one region of interest which are not meaningful in practice. As a result, if there does not exist at least one sequence $\bbe$, i.e., at least one infinite path in $\ccalG$ that starts from the initial state and traverses at least one accepting edge infinitely often, then this means that there is no path that satisfies $\phi$ (unless conditions (i)-(ii) mentioned before are violated).  

Assume that the discrete controller selects NBA states as discussed in Appendix~\ref{appendix:planning}. To show that the discrete controller is complete, it suffices to show that it can generate a infinite sequence of edges $\bbe$ as defined before. Note that the discrete controller selects next NBA states that the robot should reach in the same way as discussed before. Also, by assumption, the environmental structure and the continuous-time controller ensure that at least one of the candidate next NBA states (i.e., the ones that can decrease the distance to $\ccalV_F$) can be reached. Based on these two observations, we conclude that such a sequence $\bbe$ will be generated, 
completing the proof.
\end{proof}


Note that the graph $\ccalG$ is agnostic to the structure of the environment, meaning that an edge in $\ccalG$ may not be able to be traversed. For instance, consider an edge in this graph that is enabled only if the robot applies a certain action to a movable object that is in a region blocked by fixed obstacles; in this case the continuous-time controller will not be able to execute this action due to the environmental structure. Satisfaction of the second assumption in Proposition~\ref{proposition:completeness} implies that if such scenarios never happen, (e.g., all regions and objects that the robot needs to interact with are accessible and the continuous-time controller allows the robot to reach them) then the proposed hybrid control method will satisfy the assigned LTL task if this formula is feasible. However, if the second assumption does not hold,  there may  be an alternative sequence of automaton states to follow in order to satisfy the LTL formula that the proposed algorithm failed to find due to the \`a-priori unknown structure of the environment.

%% file: appendix_reactive_planner.tex
\section{Reactive Controller Overview}
\label{appendix:reactive_controller_overview}

This Appendix provides a brief description of the reactive, vector field controller from \cite{vasilopoulos_pavlakos_bowman_caporale_daniilidis_pappas_koditschek_2020} used in this work. As shown in Fig.~\ref{fig:simulation_simple}, the robot navigates the physical space and discovers obstacles (e.g., using the semantic mapping engine in \cite{Bowman2017}), which are dilated by the robot radius and stored in the semantic space. Potentially overlapping obstacles in the semantic space are subsequently consolidated in real time to form the mapped space. A change of coordinates $\diffeogeneric$ from this space is then employed to construct a geometrically simplified (but topologically equivalent) model space, by merging familiar obstacles overlapping with the boundary of the enclosing freespace to this boundary, deforming other familiar obstacles to disks, and leaving unknown obstacles intact. It is shown in \cite{vasilopoulos_pavlakos_bowman_caporale_daniilidis_pappas_koditschek_2020} that the constructed change of coordinates $\diffeo$ between the mapped and the model space, for a given index set $\hybridmode$ of instantiated familiar obstacles, is a $C^\infty$ diffeomorphism away from sharp corners. Using the diffeomorphism $\diffeo$, we construct a hybrid vector field controller (with the modes indexed by $\hybridmode$, i.e., depending on external perceptual updates), that guarantees simultaneous obstacle avoidance and target convergence, while respecting input command limits, in unexplored semantic environments (see \cite[Eq. (1)]{vasilopoulos_pavlakos_bowman_caporale_daniilidis_pappas_koditschek_2020}). The reactive controller is further extended to accommodate differential-drive robot dynamics, while maintaining the same formal guarantees.

\section{Description of the Topology Checking Algorithm}
\label{appendix:topology_checking_algorithm}

This Appendix includes an algorithmic outline of the topology checking algorithm from Section~\ref{subsec:topology_checking}, shown in Algorithm~\ref{algorithm:topology_checking}. This algorithm works as follows. Starting with the initially assumed polygonal enclosing freespace $\enclosingfreespace$ for either the robot or the robot-object pair, we subtract the union of all known dilated movable objects in $\movableobjectsetdilated$ and fixed obstacles in $\knownobstaclesetdilated_\hybridmode$ (corresponding to the index set $\hybridmode$ of localized familiar obstacles), using standard logic operations with polygons (see e.g., \cite{elgindy_1993,clementini_1993,douglas_1973}). This operation results in a list of freespace components, which we denote by $\mathcal{L}_\freespace := (\freespace_1, \freespace_2, \ldots)$. From this list, we identify the freespace $\freespace$ as the freespace component $\freespace_k$ that contains the robot position $\robotposition$ (or the robot-object pair center $\robotposition_{i,c}$) and re-define the enclosing freespace as its convex hull, i.e., $\enclosingfreespace := \text{Conv}\left(\overline{\freespace}_k\right)$.

If the goal $\goalposition$ is contained in $\enclosingfreespace$, the reactive controller proceeds as usual, using $\enclosingfreespace$ for the diffeomorphism construction (see Appendix~\ref{appendix:reactive_controller_overview}), and treating all other freespace components as obstacles. Otherwise, we need to check whether movable objects or fixed obstacles cause a freespace disconnection that does not allow for successful action completion. Namely, we need to check whether both the robot position $\robotposition$ (or the robot-object pair center $\robotposition_{i,c}$) and the target $\goalposition$ are included in the same connected component of the set $\mathcal{L}_{\freespace+\movableobjectsetdilated} := \left(\bigcup_i \freespace_i\right) \cup \left(\bigcup_j \movableobjectdilated_j\right)$, i.e., the union of all freespace components in $\mathcal{L}_\freespace$ with all dilated movable objects in $\movableobjectsetdilated$. This would imply that a subset of movable objects in $\movableobjectsetdilated$ blocks the target configuration. In that case, the robot switches to the Fix mode to rearrange these objects; otherwise, the interface layer reports to the symbolic controller that the current action is infeasible.

In the former case, we proceed one step further to identify the blocking movable objects in order to reconfigure them on-the-fly. First, we isolate the connected components of the union of all movable objects in $\movableobjectsetdilated$ into a list $\mathcal{L}_{\movableobjectsetdilated} := (\movableobjectsetdilated_1, \movableobjectsetdilated_2, \ldots)$; we refer to the elements of that list as the {\it movable object clusters}. Assuming that each movable object cluster is connected to at most two freespace components from $\mathcal{L}_\freespace$, we build a connectivity tree rooted at the robot's (or the robot-object pair's) freespace $\freespace$, by checking whether the closures of two individual regions overlap; the tree's vertices are geometric regions (freespace components in $\mathcal{L}_\freespace$ and movable object clusters in $\mathcal{L}_{\movableobjectsetdilated}$) and edges denote adjacency. We then backtrack from the vertex of the tree that contains the goal $\goalposition$ until we reach the root, saving the encountered movable object clusters along the way. Any movable object intersecting any of these clusters is pushed to a stack of {\it blocking movable objects} $\mathcal{B}_{\movableobjectsetdilated}$.

\begin{algorithm}
\begin{algorithmic}
\Function{TopologyChecking}{$\robotposition$,$\goalposition$,$\enclosingfreespace$,$\movableobjectsetdilated$,$\knownobstaclesetdilated_\hybridmode$}
\State $\mathcal{L}_\freespace \gets \texttt{Subtract}(\enclosingfreespace,\texttt{Union}(\movableobjectsetdilated,\knownobstaclesetdilated_\hybridmode$))
\For{$\freespace_k \in \mathcal{L}_\freespace$}
\If{$\robotposition \in \freespace_k$}
\State $\freespace \gets \freespace_k$
\State $\enclosingfreespace \gets \text{Conv}(\overline{\freespace}_k)$
\State $\textbf{break}$
\EndIf
\EndFor
\If{$\goalposition \in \freespace$}
\State $\mathcal{B}_\movableobjectsetdilated \gets \varnothing$ \Comment{No blocking objects or obstacles}
\State $\texttt{IsFeasible} \gets \texttt{True}$ \Comment{Task feasible}
\Else
\State $\mathcal{L}_{\freespace+\movableobjectsetdilated} \gets (\bigcup_i \freespace_i) \cup (\bigcup_j \movableobjectdilated_j), \mathcal{F}_i \in \mathcal{L}_\freespace, \movableobjectdilated_j \in \movableobjectsetdilated$
\For{$\freespace_k \in \mathcal{L}_{\freespace+\movableobjectsetdilated}$}
\If{$\robotposition \in \freespace_k$}
\If{$\goalposition \in \freespace_k$}
\State $\mathcal{L}_\movableobjectsetdilated \gets \bigcup_j \movableobjectdilated_j, \movableobjectdilated_j \in \movableobjectsetdilated$
\State $(\mathcal{V}_\mathcal{L},\mathcal{E}_\mathcal{L}) \gets \texttt{ConnectTree}(\mathcal{L}_\freespace,\mathcal{L}_\movableobjectsetdilated)$
\For{$V \in \mathcal{V}_\mathcal{L}$}
\If{$\goalposition \in V$}
\State $\mathcal{B}_\movableobjectsetdilated \gets \texttt{BacktrackFrom}(V)$
\State $\textbf{break}$
\EndIf
\EndFor
\State $\texttt{IsFeasible} \gets \texttt{True}$
\Else
\State $\mathcal{B}_\movableobjectsetdilated \gets \varnothing$ \Comment{Blocked by fixed obstacles}
\State $\texttt{IsFeasible} \gets \texttt{False}$
\EndIf
\State $\textbf{break}$
\EndIf
\EndFor
\EndIf
\State $\textbf{return} \quad \enclosingfreespace, \texttt{IsFeasible}, \mathcal{B}_\movableobjectsetdilated$
\EndFunction
\end{algorithmic}
\caption{Topology Checking Algorithm.} \label{algorithm:topology_checking}
\end{algorithm}